\documentclass[acmsigecom,final]{style/sigecom/acmtrans2m}
\acmVolume{24}
\acmNumber{1}
\acmYear{26}
\acmMonth{June}
\usepackage{microtype}

\usepackage{amsmath,amsfonts,amssymb,mathtools}

\allowdisplaybreaks[1]

\usepackage{graphicx}
\usepackage{booktabs}
\usepackage{tabularx}
\usepackage{array}
\usepackage{multirow}
\usepackage{adjustbox}
\usepackage{bbm}
\usepackage{xcolor}
\usepackage{tikz}
\usepackage{pgfplots}
\pgfplotsset{compat=1.18}

\usepackage{listings}
\usepackage[most]{tcolorbox}
\tcbuselibrary{skins,breakable,listings}

\makeatletter
\let\bibhang\@undefined
\makeatother
\usepackage[numbers,sort&compress]{natbib}

\makeatletter
\let\pdfbookmark\@undefined
\makeatother
\usepackage[pagebackref=true]{hyperref}
\usepackage{hypernat}
\definecolor{mylinkblue}{HTML}{0072bc}
\hypersetup{
  colorlinks=true,
  allcolors=mylinkblue,
  pdfencoding=auto,
  unicode=true,
}
\makeatletter
\renewcommand*{\NAT@nmfmt}[1]{\hyper@natlinkstart{\@citeb}#1\hyper@natlinkend
}
\makeatother
\usepackage[capitalize,noabbrev,nameinlink]{cleveref}
\crefformat{equation}{#2Equation~#1#3}
\Crefformat{equation}{#2Equation~#1#3}

\input{style/macros.tex}
\tightdisplayspacing

\makeatletter
\renewcommand{\@biblabel}[1]{[#1]}
\makeatother

\newcommand{\papertitle}{AI Alignment From Social Choice Perspectives}

\title[\papertitle]{\papertitle}

\author[HALPERN et al.]{DANIEL HALPERN \\ Google Research\\
EVI MICHA \\ University of Southern California\\
ARIEL D. PROCACCIA \\ Harvard University\\
BENJAMIN SCHIFFER \\ Harvard University\\
ITAI SHAPIRA \\ Harvard University
\and
SHIRLEY ZHANG \\ Harvard University
}

\begin{abstract}
Alignment from human feedback uses human judgments about model outputs to steer the behavior of language models after pretraining. When those judgments reflect conflicting views of desirable behavior, the learned objective becomes an aggregate determination of what the model should prefer. We survey recent work that has studied this aggregation problem through the lens of social choice theory. We illustrate how the social choice perspective helps identify failure modes in the feedback aggregation layer and reveals a broader design space for handling disagreement in explicit and principled ways.
\end{abstract}

\begin{document}
\setcounter{page}{47}

\makeatletter
\def\pages{\pageref{@firstpg}--\pageref{@lastpg}}
\def\mypage{\thepage}
\def\runningfoot{\gdef\@runningfoot{\@journalName, Vol.\ \@acmVolume, No.\ \@acmNumber, \@acmMonth\ 20\@acmYear, Pages \pageref{@firstpg}--\arabic{finalpage}}}
\def\firstfoot{\gdef\@firstfoot{\@journalName, Vol.\ \@acmVolume, No.\ \@acmNumber, \@acmMonth\ 20\@acmYear, Pages \pageref{@firstpg}--\arabic{finalpage}}}
\let\survey@abstractbottom\@abstractbottom
\def\@abstractbottom{\vskip-\parskip\survey@abstractbottom}
\makeatother

\begin{bottomstuff}
Authors are listed alphabetically. Correspondence to Itai Shapira
(\href{mailto:itaishapira@g.harvard.edu}{itaishapira@g.harvard.edu}).
\end{bottomstuff}

\maketitle

\section{Introduction}\label{sec:intro}
\emph{AI alignment} is the problem of ensuring that artificial intelligence systems act in ways consistent with human intentions, preferences, and normative constraints~\citep{leike_scalable_2018,ji_ai_2025,gabriel_artificial_2020,russell_human_2019}. Achieving alignment requires identifying which outputs and behavioral patterns are desirable or unacceptable, and using those judgments to control and steer the system~\citep{taylor_alignment_2020,kenton_alignment_2021,gabriel_matter_2025}.

For open-ended language models, operationalizing these goals presents a difficult challenge. The full range of desirable behaviors cannot be explicitly specified~\citep{hendrycks_unsolved_2022}; the underlying human judgments rely on tacit knowledge and subtle tradeoffs between objectives, making them virtually impossible to capture in formal terms~\citep{leike_scalable_2018,casper_open_2023}.

Fortunately,  humans are far better at recognizing acceptable behavior than they are at formally articulating it~\citep{clark_why_2018,yue_k-armed_2012}, an asymmetry that motivates a shift from explicit specification to learned evaluation~\citep{NIPS2017_d5e2c0ad}. In methods for \emph{alignment from human feedback}, most notably \emph{reinforcement learning from human feedback (RLHF)}~\citep{NIPS2017_d5e2c0ad,ouyang_training_2022,ziegler_fine-tuning_2020,stiennon_learning_2020,rafailov_direct_2023}, human annotators evaluate concrete model outputs, and a scoring function trained on this feedback is then used to align the model. Typically, annotators express preferences via pairwise comparisons, selecting the better of two responses to a given prompt.  Aggregated across users and contexts, the annotations are distilled into a learned reward model: a parametric scoring function trained to predict the desirability of a given output. This reward model is expected to generalize beyond the labeled training data, capturing implicit norms and tradeoffs among competing objectives~\citep{jacobs_measurement_2021}.

The foregoing methods compress human preferences about model behavior into a single, universally applicable scalar reward signal that ostensibly represents human judgment. This approach implicitly assumes a shared, underlying human intuition that can be statistically recovered by querying human evaluators~\citep{sorensen_position_2024,halpern_pairwise_2025,zhi-xuan_beyond_2025}. Under this view, annotators are interchangeable. Conflicting preferences are treated as noisy observations of a common ground truth, rather than evidence of value pluralism worth reflecting in the system~\citep{sorensen_value_2024,halpern_pairwise_2025}. While such methods have proven empirically successful on tasks where evaluators largely agree~\citep{ouyang_training_2022,baiTrainingHelpfulHarmless2022}, in many contexts, interpretations of ``correct'' behavior diverge across backgrounds and cultures~\citep{yuan2024cultural,prabhakaran2022cultural,anwarFoundationalChallengesAssuring2024a,kirk_prism_2024,feng2024modular,zhang_diverging_2025}. 

When annotators inherently disagree, reward modeling goes beyond statistical estimation; it becomes a form of preference aggregation. It collapses conflicting individual judgments into a single collective preference, implicitly fixing tradeoffs among competing values while obscuring the mechanism used to resolve them~\citep{pardeshi_learning_2024}.\footnote{Constitutional AI~\citep{baiConstitutionalAIHarmlessness2022} and its variants inherit the same framing one level up, as the ``constitution'' is itself the product of preference aggregation over normative principles~\citep{huang_collective_2024}, and discretion must still be exercised in resolving their conflicts~\citep{buyl_ai_2025}.} Evaluating this choice requires examining the assumed model structure and the aggregation method itself. These questions cannot be resolved within the training objective alone. Instead, they fall squarely within the domain of \emph{social choice theory}~\citep{sen_possibility_1998,sen_social_1986,conitzer_position_2024,dai_mapping_2024,mishra_ai_2023,ge_axioms_2024,alamdariPolicyAggregation2024,maura-rivero_jackpot_2025,conitzer_moral_2017,prasad_social_2018}, the mathematical study of how heterogeneous individual preferences are aggregated into collective decisions.

Viewed through this lens, alignment can be analyzed as a formal aggregation pipeline: from elicited judgments to learned objectives to optimized policies.
Social choice provides the precise language to express desired properties explicitly, making the normative assumptions embedded in aggregation procedures more transparent and open to mathematical comparison.

In this paper, we survey three related types of contributions the social choice perspective has brought to recent alignment research. First, this framing identifies the aggregation rules embedded in widely used methods such as RLHF and DPO~\citep{rafailov_direct_2023}, showing which assumptions and normative priorities these methods build into their objectives. Second, social choice axioms expose structural failure modes of such rules, identifying conditions under which alignment methods provably violate desirable aggregation properties. Third, social choice provides a broad toolbox of well-studied aggregation methods, informing new approaches that encode different normative objectives through various choices of feedback elicitation, aggregation, and policy optimization.

Yet thinking of alignment from human feedback purely in terms of classical social choice misses two distinctive aspects of alignment that shape much of the research surveyed below. The first is \emph{generalization}. Because the candidate-response space is generated and effectively unbounded, the feedback data cover only a sparse set of comparisons over prompt-response pairs, provided by a small subset of evaluators. The learned preference model must then generalize in three directions: from evaluated responses to unseen responses, from curated prompts to unseen prompts, and from sampled annotators to a broader population of users~\citep{gao_scaling_2023}.

The second feature is that alignment is ultimately evaluated on the \emph{downstream policy}~\citep{wen_rethinking_2024,frickHowEvaluateReward2024a}. The intermediate learned reward model matters only insofar as optimizing against it produces a policy aligned with human judgments. A reward model that satisfies desirable aggregation properties need not induce a policy that does the same.

The goal, then, is to adapt what social choice has long understood about collective decision-making into the practical machinery of model training. AI alignment, in this light, is not merely the problem of steering AI to follow human values, but rather of establishing fair principles for incorporating the diversity of values people actually hold~\citep{gabriel_artificial_2020,gabriel_challenge_2023,baum_social_2020}. The sections that follow trace this adaptation. 

\runinhead{Survey Roadmap.}
After \roadmapref{sec:rlhf}{RLHF Primer and Notation} introduces the notation used throughout this paper, \roadmapref{sec:bt-borda}{Reward Learning as Implicit Aggregation} identifies the aggregation rule implicit in unconstrained Bradley--Terry reward learning, showing that it behaves like the classic \emph{Borda rule}.

The next two sections examine failure modes that surface when this aggregation rule is placed under the distinctive constraints of alignment. \roadmapref{sec:clone-robustness}{Clone Robustness} explores clone robustness, an inherited
social-choice pathology that is especially natural in language-model settings,
where generated candidate sets may contain many nearby strings expressing the same
substantive answer, making the learned reward sensitive to how that answer is
represented in the sample. \roadmapref{sec:pareto-optimality}{Parametric Reward Learning and Unanimity} turns to failures that arise from the way the reward function is learned from data, in particular from restricting the reward to a limited reward class.

The following sections move from aggregation to the information and
policy-level limits of alignment from feedback. \roadmapref{sec:welfare-loss}{Welfare Loss and Sparse Elicitation} studies what sparse feedback can reveal about the underlying population preference distribution and how welfare can be lost when policies are optimized from limited preference information. \roadmapref{sec:nlhf}{Direct Alignment From Pairwise Preferences} then considers policy-level methods that avoid first compressing pairwise preferences into a scalar reward. It explains how \emph{Nash learning from human feedback} optimizes directly against pairwise preferences and achieves optimal welfare guarantees.
We conclude in \Cref{sec:discussion} by discussing related directions at the boundary between social choice and alignment.
\section{RLHF Primer and Notation}
\label{sec:rlhf}

This section fixes the notation used throughout the paper and reviews the standard RLHF pipeline at the level needed for our analysis. We first define prompts, responses, annotators, policies, and reward models, then describe how pairwise preference data are aggregated into a scalar reward via \emph{random utility models} and the Bradley--Terry loss, and finally recall the KL-regularized policy objective used to optimize a language model against the learned reward.

\runinhead{Setup and Notation.}
Let $\mathcal{X}$ and $\mathcal{Y}$ denote the prompt and response spaces, respectively. A prompt $x$ may encode either a single query or the dialogue history of a multi-turn interaction. Let $\annotators$ denote the target population of annotators, and let $i \sim \annotators$ denote a randomly sampled annotator.
A stochastic policy $\pi$ assigns to each prompt $x\in\X$ a distribution $\pi(\cdot\mid x)\in\Delta(\Y)$ over responses.
 We let $\pibase$ denote a fixed reference policy, typically a pretrained model that has undergone supervised fine-tuning. A reward model is a function $\rewardfn:\X\times\Y\to\mathbb{R}$ assigning a real-valued score to each prompt-response pair. When considering a parameterized family of reward models, we write $\rewardfntheta$ for the model indexed by $\theta\in\Theta$. When the prompt $x$ is fixed or clear from context, we write $\reward{y}$ as shorthand for $\reward{x,y}$, and similarly write $\reward[\theta]{y}$ for $\reward[\theta]{x,y}$. For analyses restricted to a fixed prompt $x$, let $\Y_x\subseteq\Y$ denote the set of candidate responses considered for that prompt. Finally, let $D$ denote the distribution over prompts used for post-training and policy evaluation.

\runinhead{Preference Data and Reward Learning.}\label{sec:reward_learning_setup}
Alignment from human feedback can use several forms of supervision, including demonstrations, scalar ratings, critiques, natural-language feedback, and comparisons~\citep{kaufmann_survey_2023}. In the reward-modeling stage of standard RLHF pipelines, the data are typically comparison labels: for a prompt $x$ and two candidate responses $y,y'\in\Y$, an annotator indicates which response they prefer~\citep{stiennon_learning_2020,ouyang_training_2022}. Rankings over a slate of responses are often converted into induced pairwise comparison labels~\citep{ouyang_training_2022,zhu_principled_2023}. We treat these labels as samples from a population preference relation and write
\[
P_x(y \succ y')\coloneqq \Pr_{i\sim\annotators}[y\succ_i y']\in[0,1].
\]
That is, \(P_x(y \succ y')\) is the population probability that a randomly drawn annotator prefers \(y\) to \(y'\) on prompt \(x\). Most methods considered in this survey take only this pairwise preference object as input.\footnote{In social choice theory, such rules are called \emph{C2 rules}~\cite{fishburn_condorcet_1977}.} At times, we view a sampled annotator \(i\sim\annotators\) as having an unobserved implicit reward function \(r_i(x,\cdot)\) over responses, which induces the preference relation \(\succ_i\). 

Reward learning fits a scalar function $\rewardfntheta$ whose induced pairwise
probabilities approximate these population probabilities. In standard alignment
pipelines \citep{NIPS2017_d5e2c0ad,ziegler_fine-tuning_2020,stiennon_learning_2020,ouyang_training_2022},
this is done by fitting a parametric reward model via maximum likelihood under a
\emph{random utility model (RUM)}, in which observed preferences arise from latent
per-response rewards. Specifically, 
\[P_x(y \succ y')=F\!\left(r(x,y)-r(x,y')\right),\]
where $F:\mathbb{R}\to(0,1)$ is an increasing link function satisfying
$F(t)=1-F(-t)$ and encoding the comparison noise
\citep{thurstone_law_1927,luce_individual_1959,mcfadden_conditional_1973,noothigattu_axioms_2020}.
Different noise structures, i.e., different choices of $F$, lead to different
objective functions~\citep{knox_models_2023}.
The widely used \emph{Bradley--Terry (BT) model}~\citep{bradley_rank_1952} takes
$F$ to be the sigmoid $\sigma(t)\coloneqq(1+e^{-t})^{-1}$, corresponding to
Gumbel-distributed noise~\citep{truong_machine_2026,yellott_relationship_1977}.

Concretely, let $\D$ be a dataset of pairwise comparisons, where each triple $(x,y^+,y^-)\in\D$ has $x\in\X$, $y^+,y^-\in\Y$, and records that an annotator preferred $y^+$ to $y^-$ on prompt $x$. For a parametric reward model $\rewardfntheta$, BT training minimizes the following logistic loss: 
\begin{equation}
\label{eq:bt-loss}
\mathcal{L}_{\mathrm{BT}}(\theta)
=\sum_{(x,y^+,y^-)\in\D}\log\!\Big(1+\exp\big(-(\reward[\theta]{x,y^+}-\reward[\theta]{x,y^-})\big)\!\Big).
\end{equation}
This loss function is the negative log-likelihood of the observed preference data under the assumption that preferences satisfy 
$P_x(y^+ \succ y^-)
=
\sigma\left(r(y^+) - r(y^-)\right).$
When this holds, we say that the BT model is \emph{correctly specified}. 

\runinhead{KL-Regularized RLHF.}
During post-training, the learned reward function provides scalar feedback to the language model to optimize against. In the classical formulation~\citep{ziegler_fine-tuning_2020}, the policy objective is given by: 
\begin{equation}
\label{eq:kl-obj}
\max_{\pi}\ 
\E_{x\sim D}\left[
\E_{y\sim \pi(\cdot\mid x)}\big[\reward{x,y}\big]
- \beta^{-1}
\mathrm{KL}\big(\pi(\cdot\mid x)\,\|\,\pibase(\cdot\mid x)\big)
\right].
\end{equation}
We abbreviate
\[
\mathrm{KL}_{D}(\pi \,\|\, \pibase)
\coloneqq
\E_{x\sim D}\!\left[\mathrm{KL}\big(\pi(\cdot\mid x)\,\|\,\pibase(\cdot\mid x)\big)\right].
\]
The second term keeps the new policy from drifting too far from the reference policy, and we treat $\beta$, the tilt strength (inverse temperature), as a control parameter for this drift. Larger $\beta$ pushes $\pi(\cdot\mid x)$ more aggressively toward high-reward responses and further from $\pibase(\cdot\mid x)$. Even a small gap between two responses can then shift a large amount of probability mass when $\beta$ is large, so aggregation errors that look minor at the reward-learning stage can be magnified by downstream policy optimization.

\section{Reward Learning as Implicit Aggregation}
\label{sec:bt-borda}
 
When human feedback is heterogeneous, different annotators can give conflicting preference labels for the same prompt. Fitting a single reward then requires the objective to reconcile these judgments into one score per response. BT reward learning is usually read as statistical estimation of latent utilities from noisy pairwise comparisons. However, this framing obscures the reconciliation process, treating it as a technical detail rather than a choice of aggregation rule with inherent normative tradeoffs. This section asks what rule BT reward learning implicitly applies, and shows it to be closely connected to a classical voting rule.

The \emph{Borda rule}, also known as Borda count, is named after Jean-Charles de Borda, who (re)introduced it in the 18th century~\citep{Borda1784,brandt2016handbook}. In our setting, fix a prompt \(x\) and let \(D_x\) be the distribution over responses used to generate comparison candidates. Let each annotator induce a preference relation \(\succ_i\) over responses in the support of \(D_x\).  The \emph{Borda score} of a response \(y\) is the population probability that \(y\) is preferred to a randomly drawn comparison candidate:\footnote{If ties are allowed, one can replace \(P_x(y \succ y')\) by \(P_x(y \succ y')+\tfrac12 P_x(y \sim y')\). If \(D_x\) assigns positive probability to \(y'=y\), this convention also treats self-comparisons as ties.}
\begin{equation}\label{eq:borda}
    \B(y)\coloneqq\E_{y'\sim D_x}\bigl[P_x(y \succ y')\bigr].
\end{equation}
That is, Borda favors responses that maximize their expected pairwise advantage against a randomly drawn candidate.

In practice, the underlying aggregation rule of BT reward learning is obscured by two structural constraints: the learner observes only a sparse, finite sample of comparisons, and it must fit a parametric reward function to generalize across distinct prompts and responses. To isolate the aggregation step, we remove these constraints and consider an idealized setting. Assume the true population pairwise preference probabilities are known for all pairs in the candidate set $\mathcal{Y}_x$, and allow each response $y \in \mathcal{Y}_x$ to receive an independent, arbitrary scalar reward. The BT objective reduces to the prompt-level population loss
\begin{equation}\label{eq:bt-pop}
    \mathcal{L}_x(\rewardfn)
    \coloneqq
    -\,\E_{y,y'\sim D_x}\bigl[\,P_x(y \succ y')\,\log\sigma\bigl(\reward{y}-\reward{y'}\bigr)\bigr],
\end{equation}
minimized over all $\rewardfn \colon \operatorname{supp}(D_x) \to \mathbb{R}$. This represents the asymptotic limit of empirical BT learning given infinite comparisons and a fully expressive reward class. Under these conditions, the BT objective has a clean social-choice characterization: its optimizer ranks responses exactly by their Borda scores.
\begin{theorem}[\cite{anderson_relationships_2009}] \label{thm:bt_borda}
Fix a prompt \(x\), and let \(D_x\) be the distribution over responses
used to generate comparison candidates. Let \(\rewardfnstar\) be a finite-valued minimizer of \cref{eq:bt-pop}. Then, for any
\(y,y'\in\operatorname{supp}(D_x)\),
\[
\B(y)>\B(y')
\quad\Longleftrightarrow\quad
\rewardstar{y}>\rewardstar{y'}.
\]
\end{theorem}
\begin{proof}[of \Cref{thm:bt_borda}]
For any \(y\) with \(D_x(y)>0\), differentiating the BT objective with respect to
\(\reward{y}\) gives
\[
\frac{\partial\mathcal{L}_x}{\partial \reward{y}}(\rewardfn)
=
D_x(y)
\E_{z\sim D_x}
\Big[ \sigma(\reward{y}-\reward{z}) - P_x(y\succ z) \Big].
\]
Thus first-order optimality at \(\rewardfnstar\), together with \(D_x(y)>0\), gives
\begin{equation}\label{eq:bt-borda-row-sum}
\B(y)
=
\E_{z\sim D_x}
\Big[
\sigma(\rewardstar{y}-\rewardstar{z})
\Big].
\end{equation}
Let
\(
f(t)=
\E_{z\sim D_x}
[
\sigma(t-\rewardstar{z})
].
\)
By \cref{eq:bt-borda-row-sum}, \(\B(y)=f(\rewardstar{y})\) for every
\(y\) with \(D_x(y)>0\). Since \(\sigma(t-\rewardstar{z})\) is strictly
increasing in \(t\) for every \(z\), the function \(f\) is strictly increasing.
Hence, for all \(y,y'\) with \(D_x(y),D_x(y')>0\),
\begin{align*}
  \B(y) > \B(y')
  \ \Longleftrightarrow\ 
  f(\rewardstar{y}) > f(\rewardstar{y'})
  \ \Longleftrightarrow\ 
  \rewardstar{y} > \rewardstar{y'}. \tag*{\qed}
\end{align*}
\end{proof}
The lineage of this result is long and somewhat fragmented. The identity in \cref{eq:bt-borda-row-sum} appears already in the work of \citet{zermelo_berechnung_1929}, was independently rediscovered by \citet{bradley_rank_1952} and \citet{Ford1957SolutionOA}, and was restated by \citet{daniels_round-robin_1969} and \citet{Jech_1983}, all assuming that the pairwise preference probabilities are generated by a BT model. \citet{anderson_relationships_2009} state the result in the form closest to ours. \citet{siththaranjan_distributional_2023} restate it in the RLHF setting.

\cref{thm:bt_borda} clarifies what kind of disagreement BT reward learning preserves when it approximates Borda-style aggregation. Intuitively, Borda favors breadth of acceptability rather than depth of support or pairwise dominance \citep{brandt2016handbook}. A response can receive a high score because it is consistently acceptable against many candidate responses, even when another response has stronger support in a particular direct comparison. In this sense, Borda selects the consensus or compromise response, the one no annotator subgroup strongly objects to, over a response backed by a narrower majority, and it downranks polarizing responses \citep{maskin_bordas_2025,reilly_social_2002}. The classic critique is that broad mid-strength acceptability can systematically advantage bland responses over distinctive ones that elicit both strong support and strong opposition in the annotator population \citep{balinski_majority_2010}.

After the reward model is learned, post-training converts this score into a change in the model's response distribution, shifting probability mass toward responses with higher Borda scores. How strongly any particular response is upweighted depends on both the Borda score learned from annotators and the base policy $\pibase$. To make this concrete, we return to the idealized setting and assume that fine-tuning attains the unparameterized optimum, where the decision variable for each prompt $x$ is the conditional distribution $\pi(\cdot \mid x)$ itself. The maximizer then has the closed-form Boltzmann/Gibbs expression~\citep{todorov_linearly-solvable_2006,peters_relative_2010}
\[\pi^\star(y \mid x) \;\propto\; \pibase(y \mid x)\,\exp\bigl(\beta\, r(y)\bigr),\]
from which it follows~\citep{shirali_direct_2025,truong_machine_2026} that:
\[
\B(y)>\B(y')
\quad\Longleftrightarrow\quad
\frac{\pi^\star(y\mid x)}{\pibase(y\mid x)}
>
\frac{\pi^\star(y'\mid x)}{\pibase(y'\mid x)}.
\]
This implies that the post-training density ratio against the base policy is, up to a monotone transformation, the \(\pibase\)-weighted Borda score of \(y\).

In practice, many prominent alignment deployments intentionally flatten $P_x$ into binary preference labels by taking the majority vote, disregarding the soft-label information as noise~\citep{kirk2023past,chen_preference_2024,kim2024margin,lambert2023history,prabhakaran2021releasing,ganguliRedTeamingLanguage2022,stiennon_learning_2020,baiConstitutionalAIHarmlessness2022}. Applied before BT fitting, this preprocessing induces the \emph{Copeland rule}~\citep{copeland_reasonable_1951}, which ranks responses by the number of pairwise majority contests they win~\citep{xiao_theoretical_2025,an_differential_2026}.

\section{Clone Robustness}\label{sec:clone-robustness}

\Cref{sec:bt-borda} identified BT reward learning as a Borda-like aggregation rule on a fixed candidate set. This connection suggests that BT reward learning may inherit some of Borda's pathologies. A particularly relevant one for language-model reward learning is sensitivity to near-duplicate candidates. In voting theory, this means that adding candidates nearly identical to an existing candidate can change the rule's outcome~\citep{tideman_independence_1987}. Since Borda is well known to be sensitive to such near-duplicate candidates, \emph{robustness to approximate clones} becomes a natural first benchmark for BT reward learning.

Approximate clones arise structurally in generated language~\citep{berker_obvious_2024,conitzer_position_2024}. For a fixed prompt, candidate responses are sampled from a generative model whose probability mass often concentrates around a small number of semantic and stylistic modes. A single substantive answer can therefore appear in many forms, with differences that are largely surface-level, such as paraphrasing, reordered explanations, or changes in verbosity. The candidate set $\Y_x$ is only a sampled discretization of this broader response space, and the often arbitrary number of variants associated with a given substantive answer reflects the sampling and filtering process as much as human preference. Exact copies can often be extracted and merged before training, but the harder case involves near-duplicates: responses that differ slightly while occupying the same semantic region and expressing the same underlying answer.

Within a near-clone cluster, pairwise comparisons among variants provide little new signal. Annotators may exhibit minor idiosyncratic preferences among them, but these differences are usually not the underlying quality the reward model is meant to capture. A clone-sensitive BT model can nevertheless treat the number of variants in the cluster as meaningful, so adding near-duplicates can change the rewards assigned to individual candidates and the total policy probability assigned to the cluster as a whole. The learned reward, and by extension the policy it induces, can then depend on arbitrary sampling artifacts, such as how densely different answers are represented in $\Y_x$, rather than on human preference~\citep{berriaud_clone-robust_2025}. As models generate more fluent paraphrases and subtle stylistic variants of the same substantive answer, this representational multiplicity can become a larger source of reward variation.

The corresponding formal requirement is \emph{robustness to approximate clones}: the learned reward should be stable under the addition of near-duplicates. Adding a response that is nearly identical to an existing one should make the two responses receive nearly the same reward, and should not substantially change the rewards assigned to unrelated responses. \citet{procaccia_clone-robust_2025} formalize this requirement using a metric on responses. Let $\rho$ be a metric on $\Y$, where smaller values indicate greater similarity. Fix a prompt $x$ and a finite candidate set $\Y_x \subseteq \Y$. Let $\rewardfn_{\Y_x}$ denote the reward learned from comparisons over $\Y_x$, for example by minimizing the unrestricted prompt-level population analogue of \cref{eq:bt-loss}.

\begin{definition}A reward-learning procedure is \emph{robust to approximate clones} if for every $\delta>0$ there exists $\varepsilon>0$ such that the following holds. For any finite candidate set $\Y_x \subseteq \Y$, any $y\in \Y_x$, and any new response $y'$ satisfying $\rho(y,y')\le \varepsilon$, the rewards learned before and after adding $y'$ satisfy
\[
    \left|\rewardfn_{\Y_x\cup\{y'\}}(y)
    -
    \rewardfn_{\Y_x\cup\{y'\}}(y')\right|
    \le \delta
\]
and
\[
    \left|\rewardfn_{\Y_x\cup\{y'\}}(z)
    -
    \rewardfn_{\Y_x}(z)\right|
    \le \delta
    \qquad \forall z\in \Y_x.
\]
\end{definition}
The first condition says that the original response and its approximate clone should receive nearly the same reward once both are present. The second says that adding the clone should not substantially change the rewards assigned to the pre-existing responses. Together, the two conditions require the reward-learning rule to treat near-duplicates as redundant representations of the same local region.

Using the connection to Borda, \citet{procaccia_clone-robust_2025} show that standard BT reward learning is not robust to approximate clones. They then propose a weighted version of the population BT objective that addresses this problem. The idea is to assign each response a uniqueness weight $w(y)$, measuring the normalized mass of points in the response space for which $y$ is the nearest candidate. Responses in crowded regions receive smaller weight, while responses representing larger regions receive larger weight. Applying these weights to pairwise comparisons yields the weighted BT objective in the following theorem.

\begin{theorem}[\cite{procaccia_clone-robust_2025}] \label{thm:declone}
    Let $\mathcal{S}\subseteq\mathbb{R}^d$ be a compact response space with finite positive volume. Fix a prompt $x$, a finite candidate set $\Y_x\subseteq\mathcal{S}$, and $\lambda>0$. Draw a point $s$ uniformly at random from $\mathcal{S}$, and assign it to one of the nearest candidates in $\Y_x$, breaking ties uniformly at random; let $w_{\Y_x}(y)$ denote the resulting weight of candidate $y$. Define the weighted BT loss 
\begin{align*}
\mathcal{L}^w_\lambda(\rewardfn)\coloneqq
-\sum_{\mathclap{\substack{y,y'\in \Y_x\\y\neq y'}}}
w_{\Y_x}(y)w_{\Y_x}(y')P_x(y \succ y')
\log \sigma(\reward{y}-\reward{y'})
+\frac{\lambda}{2}\sum_{y\in \Y_x}w_{\Y_x}(y)\reward{y}^2.
    \end{align*}
    If $P_x(y\succ y')$ is Lipschitz continuous in the responses, then the reward-learning algorithm that minimizes $\mathcal{L}^w_\lambda(\rewardfn)$ is robust to approximate clones under the $\ell_2$ metric.
\end{theorem}
\cref{thm:declone} should be read as a representation-invariance guarantee for the prompt-level aggregation step. Once an embedding and a metric are fixed, near-duplicate responses are treated as surface representatives of the same region, so adding another does not substantially change the learned rewards. The guarantee reduces the need for exact deduplication. What it does not remove is the dependence on the embedding, metric, and reference space $\mathcal S$ that fix the weights. In this sense, the invariance holds relative to a fixed geometry, shifting the design burden from detecting duplicates to specifying when responses count as close~\citep{deshpande_c-sts_2023}.

\section{Parametric Reward Learning and Unanimity}
\label{sec:pareto-optimality}
\cref{sec:bt-borda,sec:clone-robustness} studied BT reward learning under an unconstrained reward class, where each prompt-response pair has its own free scalar score. In that setting, BT reproduces the Borda score and inherits both its guarantees and its pathologies. In practice, the reward must be implemented by a single parametric function that scores all prompt-response pairs, so the score assigned to one response is tied to the scores assigned to others through the same parameters. This coupling creates a new source of failures beyond the aggregation effects already discussed.

Formally, the learner chooses parameters $\theta$ from a parameter space $\Theta$, which induce a reward function $r_\theta: \X \times \Y \to \mathbb{R}$. The possible rewards are therefore restricted to the hypothesis class $\mathcal{R}_\Theta = \{r_\theta : \theta \in \Theta\}$. In the standard BT pipeline, $\theta$ is fitted on the comparison dataset $\D$ by minimizing the logistic loss in \cref{eq:bt-loss}, or equivalently by maximizing the BT likelihood.

The simplest failure is a violation of \emph{unanimity}, also known as \emph{Pareto efficiency}. If every annotator strictly prefers response $y^+$ to response $y^-$, the learned reward should assign $y^+$ a higher score. This property is worth studying not because perfect unanimity is likely to occur in practice, but because violating it means ignoring the strongest possible consensus signal, falling below the baseline set by every standard voting rule in the classical setting.
\begin{definition}A reward-learning procedure satisfies \emph{unanimity} if, for any pair of responses $y^+, y^- \in \Y_x$, whenever every annotator prefers $y^+$ to $y^-$, the learned reward satisfies \(\rewardtheta{y^+} > \rewardtheta{y^-}\).
\end{definition}

\citet{ge_axioms_2024} established this failure for a linear hypothesis class; \citet{hollender_enforcing_2025} showed that it extends to richer parametric reward classes. We follow the linear setting of \citeauthor{ge_axioms_2024}, which admits the cleanest analysis.

In the linear model, each response $y \in \Y_x$ is associated with a known feature vector $\phi(y) \in \mathbb{R}^d$, where $d$ is the embedding dimension, and the reward is linear in those features, $\rewardtheta{y} \coloneqq \langle \theta, \phi(y) \rangle$ with $\theta \in \mathbb{R}^d$. The learner holds $\phi$ fixed and fits only $\theta$. This corresponds to one way of building a reward model, in which $\phi$ is the embedding obtained by removing the final layer of a pretrained language model and the reward is a linear head trained on preference data on top of it. In many settings, the embedding $\phi$ is itself parameterized and trained alongside the last layer rather than held fixed. The linear model described here is best read as an analytically convenient special case of the restricted reward class $\mathcal{R}_\Theta$.

Throughout this section, we fix a prompt \(x \in \X\) and a candidate set \(\Y_x \subseteq \Y\). We assume that the learner observes the complete pairwise preferences of a set of annotators \(\annotators\). That is, for every annotator \(i \in \annotators\) and every pair of distinct responses \(y,y' \in \Y_x\), the learner observes which of the two responses \(i\) prefers.

We further assume that the learner fits a loss function to these observed pairwise preferences. Specifically, the analysis in this section concerns the broader family of loss-based reward-learning rules, of which BT is one member. Given any loss function \(\ell:\mathbb{R}\to\mathbb{R}\), define
\[
\mathcal{L}(\theta;\ell)
\coloneqq
\sum_{\substack{y^+,y^- \in \Y_x\\ y^+ \neq y^-}}
n_{y^+\succ y^-}
\cdot
\ell\left(
\rewardtheta{y^-}-\rewardtheta{y^+}
\right),
\]
where $n_{y^+\succ y^-}$
denotes the number of annotators who prefer \(y^+\) to \(y^-\).  When \(\ell(z)=\log(1+\exp(z))\), this is the fixed-prompt version of the BT logistic loss in \cref{eq:bt-loss}; the corresponding population objective replaces the empirical counts \(n_{y^+ \succ y^-}\) by the pairwise preference probabilities \(P_x(y^+ \succ y^-)\). Hinge loss gives another instance, and more generally the larger the gap \(\rewardtheta{y^-} - \rewardtheta{y^+}\) on a comparison the rule disagrees with, the larger the penalty. This class of rules is a natural family that contains BT and matches the loss-minimization framing of modern training pipelines.

In the unrestricted setting of \cref{sec:bt-borda}, BT satisfies unanimity as a consequence of its equivalence with Borda. This makes unanimity appear to be a natural property of BT. Yet, the theorem below shows that this guarantee need not survive once the reward is required to generalize through a restricted reward class $\mathcal{R}_\Theta$.
\begin{theorem}[\cite{ge_axioms_2024}]
Let \(\ell:\mathbb R\to\mathbb R\) be a loss function satisfying \(\inf_z \ell(z)<\ell(0)\), and suppose that \(\ell\) is either nondecreasing and weakly convex, or strictly convex. Then there exist a dimension \(d\), a finite candidate set \(\Y_x\), feature vectors \(\phi(y)\in\mathbb R^d\), and a population for which the induced loss-based linear reward rule fails unanimity. In particular, there are responses \(y^+,y^-\in\Y_x\) such that every annotator prefers \(y^+\) to \(y^-\), but the learned reward  $r_{\theta^\star}$ ranks them in the opposite order: $\reward[\theta^\star]{y^-}>\reward[\theta^\star]{y^+}$.
\end{theorem}

At a high level, the loss must find a single set of parameters that applies to all pairwise comparisons at once. Each comparison places two demands on the fit: assigning the preferred response the higher score, and increasing confidence through a larger margin $r(y^+) - r(y^-)$. In the unrestricted setting, these demands can be satisfied independently, since the relevant reward differences can be adjusted without forcing a change in unrelated comparisons. Under a restricted hypothesis class, they interact because all margins are induced by the same trainable parameters.

The loss then chooses a global fit across all observed comparisons, allocating score separation across comparison directions. Some directions matter more for this fit than others, either because many comparisons point along them or because their feature differences have a larger effect on the loss. A unanimous comparison can still be associated with a weak direction in this global problem. The learned model can then fit the dominant comparison directions while assigning the wrong ordering to the unanimous pair. 

The same parametric restriction that lets the reward generalize beyond observed comparisons can also break the guarantees of the underlying aggregation rule. In this case, the rule and its learned approximation come apart even on a property as basic as unanimity.

\section{Welfare Loss and Sparse Elicitation}
\label{sec:welfare-loss}
One can view alignment from human feedback as a limited-information approximation procedure. For each prompt, the pipeline aims to induce the model behavior that is best for a population of annotators, while observing only a partial description of their preferences. The true objective, however, may depend on unobserved information, such as the preference intensities represented by cardinal values. Here, we examine the loss incurred due to this lack of information, starting with whether there is, in fact, a loss.  

\runinhead{Identifiability.} Fix a prompt \(x\) and a candidate set \(\Y_x\subseteq\Y\). Suppose an annotator \(i\sim \annotators\) has a reward function \(r_i(\cdot)\) over these responses. A standard goal is to maximize \emph{utilitarian welfare}, defined as the expected reward across annotators. Formally, given a reward profile $\mathbf{r} = (r_i)_{i \in \annotators}$, the welfare of a response $y$ is
  \[W_x^{\mathbf{r}}(y)\coloneqq \mathbb{E}_{i\sim\annotators}\left[r_i(y)\right].\]
For a policy \(\pi(\cdot\mid x)\in\Delta(\Y_x)\), we can similarly define its expected welfare as
\[
W_x^{\mathbf r}(\pi)
\coloneqq
\mathbb{E}_{y\sim\pi(\cdot\mid x)}
\bigl[W_x^{\mathbf r}(y)\bigr].
\]
Access to these cardinal utilities would make welfare maximization immediate: we could compute the welfare of each response and train the policy accordingly. In practice, we observe only comparisons, so the question is whether they identify enough cardinal information to choose a high-welfare response.

Assuming annotators respond according to a link function $F$, we learn only the aggregate win probabilities:
  \[P_x(y \succ y')=\mathbb{E}_{i\sim\annotators}\left[F\left(r_i(y)-r_i(y')\right)\right].\]
This formulation captures BT responses, where choices are made with probabilities proportional to $\exp(r_i(y))$, as well as deterministic responses, where annotators simply select the higher-reward option. Unfortunately, unless $F$ is linear, it is impossible to reliably identify welfare-maximizing candidates from these probabilities alone, even for a single prompt $x$ with just two responses, $y$ and $y'$. Because a non-linear link function $F$ distorts the magnitude of reward differences, one can construct two distinct reward profiles, $\mathbf{r}$ and $\mathbf{r}'$, that yield identical expected win probabilities but differ in aggregate welfare. Specifically, there exist $\mathbf{r}$ and $\mathbf{r}'$ such that:
  \[\mathbb{E}_{i\sim\annotators}\left[F\left(r_i(y)-r_i(y')\right)\right]=\mathbb{E}_{i\sim\annotators}\left[F\left(r'_i(y)-r'_i(y')\right)\right]\]
  yet
  \[\mathbb{E}_{i\sim\annotators}\left[r_i(y)-r_i(y')\right]\neq\mathbb{E}_{i\sim\annotators}\left[r'_i(y)-r'_i(y')\right].\]
  Consequently, identical pairwise observation data can arise from populations with fundamentally different preferences in terms of welfare.

\runinhead{Distortion.} While \emph{exact} identification of welfare-maximizing candidates is impossible, a learned policy may still achieve approximately high welfare using only the observed comparisons. This guarantee is formalized in social choice through the concept of \emph{distortion}~\citep{procaccia_distortion_2006,boutilier_optimal_2012}. Recent work applies this lens to quantify how much welfare can be lost when a policy is learned from observed comparisons~\citep{golz_distortion_2025,oko_distortion_2026,ge_optimized_2026}.

Assume reward functions are nonnegative and normalized with $0 \le r_i(y) \le 1$ for each annotator $i$ and response $y$. Let $\mathcal{R}_x(P_x)$ denote the set of normalized reward profiles consistent with the pairwise preference object $P_x$ at prompt $x$. For a policy $\pi(\cdot \mid x) \in \Delta(\Y_x)$, its utilitarian distortion~\citep{procaccia_distortion_2006,anshelevich_filos-ratsikas_shah_voudouris_2021} at $x$ is:
\[
\mathrm{Dist}_x(\pi;P_x)
\coloneqq
\sup_{\mathbf r \in \mathcal R_x(P_x)}
\frac{
\max_{\pi' \in \Delta(\Y_x)} W_x^{\mathbf r}(\pi')
}{
W_x^{\mathbf r}(\pi)
}.
\]
Conditional on $P_x$, this quantity measures the worst-case welfare loss that $\pi$ can incur over normalized reward profiles that could have induced $P_x$,\footnote{When $W_x^{\mathbf r}(\pi)=0$, the corresponding ratio is interpreted as $\infty$.} where the welfare loss is defined by the ratio above. Since compatibility is defined by the elicited feedback, while the chosen policy is determined by the aggregation rule and policy class, the distortion reflects all three components of the decision-making process together.

\begin{theorem}[\cite{golz_distortion_2025}] \label{thm:kl-free-bt-distortion}
Fix a prompt \(x\) and a finite response set \(\Y_x\) with
\(|\Y_x|\ge 3\). Suppose annotator comparisons are well-specified by a BT model with inverse-temperature parameter \(\eta\), so that for each annotator \(i\in\annotators\),
\[
\Pr[y\succ_i y'\mid x]
=
\sigma\left(\eta\cdot (r_i(y)-r_i(y'))\right).
\]
In other words, the link function is $F(t) \coloneqq \sigma\left(\eta \cdot t \right)$.
For each pairwise preference object \(P_x\), let \(\pi_0(P_x)\) be
the policy that BT reward learning returns once the KL penalty is removed; by
\cref{thm:bt_borda}, it places all mass on the Borda winner under
\(P_x\). Then
\[
(1-o(1))\,\eta
\;\le\;
\sup_{P_x}
\mathrm{Dist}_x\left(\pi_0(P_x);P_x\right)
\;\le\;
O(\eta^{2}),
\]
where \(P_x\) ranges over pairwise preference objects generated by a
normalized reward profile \(\mathbf r\) and a comparison-pair sampling
distribution. The upper bound holds for every such \(P_x\); the lower
bound is the rate as \(\eta\to\infty\).
\end{theorem}
The lower bound should be read against the minimax lower bound for the same
information model. \citet{golz_distortion_2025}
show that, when each annotator contributes a single comparison, every rule
mapping the comparisons to a policy incurs distortion at least
\(\bigl(\tfrac12+o(1)\bigr)\eta\) on some reward profile. Strikingly, this lower bound is achieved by a framework known as Nash learning from human feedback, which we define and analyze in \cref{sec:nlhf} (see \cref{thm:nlhf-distortion-optimality}).

\runinhead{Richer Elicitation.}
\cref{thm:kl-free-bt-distortion} demonstrates that approximating welfare is possible, though imperfect, at least under BT responses. However, utilitarian welfare is not the only valid objective in social choice. Consider two hypothetical responses $y$ and $y'$: $y$ provides a utility of $2$ to all annotators, while $y'$ provides a utility of $3$ to half the annotators and $1$ to the other half. Utilitarian welfare treats these outcomes as equivalent. By contrast, more egalitarian choices, such as \emph{Nash welfare}\,---\,the product of utilities\,---\,would strictly prefer $y$, favoring broad approval over polarizing outcomes that benefit one group at the expense of another. How, then, can we optimize for these alternative choices?

\citet{chidambaram_direct_2026} and \citet{ge_linear_2026} consider linear social choice models (similar to the one described in \Cref{sec:pareto-optimality}) where responses follow BT and deterministic link functions, respectively. Both show that even in these restricted settings, a single pairwise comparison per annotator cannot reveal which of the two responses a more egalitarian objective should prefer. Without preference intensities, a single comparison cannot separate a universally indifferent population from one that is evenly split between strong opposing preferences; in both, each candidate is preferred equally often across the population. Thus, there is no hope of optimizing more egalitarian welfare functions that would strictly prefer broad-appeal candidates.  On the other hand, under mild structural conditions, eliciting two comparisons per annotator yields enough information to completely identify the voter type distribution. This identification enables direct optimization for \emph{any} desired social welfare function.

Similarly, \citet{cherapanamjeri_learning_2025} demonstrate in a latent-utility model that moving from pairwise choices to best-of-three queries provides essentially complete identifiability under appropriate structural conditions.

Richer information can also come from passively recorded signals. For example, the time taken to provide a comparison label can carry information about preference intensities that binary labels discard~\citep{echenique_general_2025,echenique_response_2026}.

While these models rely on different structural assumptions, each points to the same fact that slightly richer elicitation can, in principle, make far more welfare objectives optimizable than pairwise comparison allows.

\section{Direct Alignment From Pairwise Preferences}\label{sec:nlhf}
Earlier sections studied alignment rules that infer a scalar reward from pairwise preference data and then optimize a policy against that reward. This scalar reward projects pairwise comparisons among responses onto one ordered axis, which forces the aggregate relation into a transitive ranking. This section asks what is lost when the aggregate relation is instead intransitive, and how alignment objectives can keep the pairwise structure directly.

Such intransitivity can arise even when every annotator is internally consistent. Averaging their preferences may produce a \emph{Condorcet cycle}, where $y_a$ is majority-preferred to $y_b$, $y_b$ to $y_c$, and yet $y_c$ to $y_a$~\citep{brandt2016handbook}. In such cases there is no \emph{Condorcet winner}, meaning no response defeats every other response by majority comparison. Any deterministic target must therefore break the cycle somewhere, motivating objectives that reason over the pairwise relation itself.

\runinhead{Maximal Lotteries.}
One way to avoid arbitrarily breaking cycles is to select a distribution over responses rather than a single response. This shifts the question from which response should win to which randomized target should be chosen. \textit{Maximal lotteries} answer this by comparing each candidate distribution against alternatives using the population majority margin: draw one response from each distribution and measure which the population prefers on average. This solution concept was introduced by \citet{kreweras1965aggregation}, developed systematically by \citet{fishburn_probabilistic_1984},\footnote{The same solution has been rediscovered independently under other names, including the ``game theory method'' in voting~\citep{rivest2010optimal} and the von Neumann winner in contextual dueling bandits~\citep{dudik_contextual_2015}.} and later applied in many settings, including AI evaluation~\citep{lanctot_evaluating_2025,khalaf_robust_2026}.

Fix a prompt \(x\) and a finite candidate set \(\Y_x\subseteq\Y\). Let
\[
  M_x(y,y')
  \coloneqq
  P_x(y \succ y') - P_x(y' \succ y)
\]
denote the \emph{majority margin}\,---\,the net pairwise preference for \(y\) over \(y'\) in the population.
A maximal lottery is a distribution \(p^\star \in \Delta(\Y_x)\) whose expected margin against every response distribution $q \in \Delta(\Y_x)$ is nonnegative:
\[
\mathbb{E}_{y\sim p^\star,y'\sim q}[M_x(y,y')]\geq 0.
\]
Since \(M_x\) is skew-symmetric, the maximal lotteries are exactly the maximin strategies of the zero-sum game with payoff matrix \(M_x\):\footnote{Existence follows from von Neumann's minimax theorem~\citep{neumann_zur_1928}. For a generic majority-margin matrix, the maximal lottery is unique, although degenerate ties can yield multiple maximal lotteries.}
\[
  p^\star
  \in
  \arg\max_{p \in \Delta(\Y_x)}
  \min_{q \in \Delta(\Y_x)}
  p^\top M_x q
  =
  \arg\max_{p \in \Delta(\Y_x)}
  \min_{q \in \Delta(\Y_x)}
  \mathbb{E}_{y \sim p,\; y' \sim q}
  \left[M_x(y,y')\right].
\]
When the majority relation is cyclic, the maximal lottery spreads mass across several responses that compete in the cycle, so the target itself records the unresolved disagreement. When a strict Condorcet winner exists, no mixing is needed and the unique maximal lottery places all mass on that response. Maximal lotteries also satisfy other desirable axiomatic properties~\citep{brandl_consistent_2016-1}. In particular, they are invariant to \emph{exact} clones (duplicate responses). However, the introduction of \emph{approximate} clones\,---\,as defined in \cref{sec:clone-robustness}\,---\,can affect the lottery, even if the approximation is arbitrarily precise.

\runinhead{Nash Learning from Human Feedback.}
\emph{Nash Learning from Human Feedback (NLHF)}~\citep{munos_nash_2024,maura-rivero_jackpot_2025} applies the same maximin idea directly at the level of policies. Instead of compressing pairwise comparisons into a scalar reward, NLHF uses the pairwise preference model as the payoff of a two-player game between policies and trains toward the equilibrium of that game.

The population preference $P_x$ induces a preference between
two policies by comparing the responses they generate,
\[
  P_x(\pi \succ \pi') \coloneqq \mathbb{E}_{y \sim \pi(\cdot \mid x),\, y' \sim \pi'(\cdot \mid x)}
  \!\left[P_x(y \succ y')\right].
\]
Thus, $P_x(\pi \succ \pi')$ is the probability that $\pi$ produces a response
preferred to one produced by $\pi'$, at prompt $x$. We write
\[
P(\pi \succ \pi')\coloneqq \mathbb{E}_{x \sim D}\!\left[P_x(\pi \succ \pi')\right]
\]
for the corresponding comparison after averaging over prompts drawn from \(D\). This induces a preference relation over policies
\[
\pi \succ \pi' \quad \Longleftrightarrow \quad P(\pi \succ \pi') > \frac{1}{2}.
\]
NLHF chooses a policy with the best worst-case comparison probability against alternative policies:
\begin{align*}
\pi^\star
\in
\arg\max_{\pi}
\min_{\pi'}
P(\pi \succ \pi').
\end{align*}
For a fixed prompt with finite candidate support and no regularization, $\pi^\star$ coincides with the maximal lottery of that prompt's margin game. In this setting, the maximal lottery can be computed by the linear program for the row player's maximin strategy. In the open-ended parametric setting, NLHF instead approaches the corresponding policy-level equilibrium through self-play, optimizing a policy against opponents generated from its own iterates. This equilibrium-seeking template underlies several recent preference-optimization methods~\citep{swamy_minimaximalist_2024,wu_self-play_2024,calandriello_human_2024,rosset_direct_2024,tiapkin_proximal_2025,heymann_adaptive_2025}.

\runinhead{Welfare Interpretation.}
The welfare-loss perspective of \cref{sec:welfare-loss} gives a precise sense in which maximal lotteries are optimal. The following result bounds their worst-case utilitarian distortion under anonymous pairwise feedback.
\begin{theorem}[\cite{golz_distortion_2025}] \label{thm:nlhf-distortion-optimality}
Suppose annotator comparisons are generated from normalized rewards by the BT model in \cref{thm:kl-free-bt-distortion}, with inverse-temperature \(\eta\). Let
\(\pi_x^{\mathrm{NLHF}}(P_x)\) be any maximal lottery of the margin game
\(M_x\). Then
\[
\sup_{P_x}
\mathrm{Dist}_x\!\left(\pi_x^{\mathrm{NLHF}}(P_x);P_x\right)
\le
\left(\frac12+o(1)\right)\eta ,
\]
where \(P_x\) ranges over pairwise preference objects generated by normalized
reward profiles and comparison-pair sampling distributions, and the \(o(1)\)
term is as \(\eta\to\infty\).
\end{theorem}
Together with the minimax lower bound of \cref{sec:welfare-loss}, \cref{thm:nlhf-distortion-optimality} shows that the maximal lottery attains the smallest possible worst-case distortion, up to lower-order terms, among rules observing anonymous BT pairwise feedback.

\runinhead{Regularization.}
As with RLHF, deployed NLHF regularizes the learned policy toward a reference
policy \(\pibase\) with a KL penalty. The regularized preference between
policies is
\[
  P_\tau(\pi \succ \pi')
  \;\coloneqq\;
  P(\pi \succ \pi')
  - \tau \, \mathrm{KL}_{D}(\pi \,\|\, \pibase)
  + \tau \, \mathrm{KL}_{D}(\pi' \,\|\, \pibase),
\]
where the KL terms penalize each policy for diverging from \(\pibase\),
symmetrically across the two players.\footnote{Viewing each player's payoff
separately, a policy receives its preference payoff minus a KL penalty for
its own divergence from \(\pibase\): \(R(\pi;\pi') = P(\pi \succ \pi') -
\tau\,\mathrm{KL}_{D}(\pi \,\|\,\pibase)\), with the symmetric expression
for \(\pi'\).} The training target remains the Nash equilibrium of this
modified game, which exists and is unique~\citep{munos_nash_2024}. 

\cref{thm:nlhf-distortion-optimality}'s optimality extends to this KL-regularized objective. Since the penalty keeps the learned policy near $\pibase$, the relevant comparison is with policies that are allowed to move the same distance from the reference. That is, the regularized equilibrium is compared with the highest welfare policy whose KL divergence from $\pibase$ is no larger than its own. Against this benchmark, NLHF retains the same optimal distortion guarantee, for any reference policy and any regularization strength~\citep{golz_distortion_2025}.
\section{Discussion}\label{sec:discussion}
Taken together, the sections above give a social-choice account of alignment from human feedback as a sequence of design choices. Human judgments must be elicited, aggregated into a model, generalized beyond the observed data, and translated into policy behavior. We conclude by discussing several related research threads that extend this pipeline view beyond the formal results surveyed earlier.

\runinhead{Preference Collapse.}
Fine-tuning a base policy toward a learned reward signal can narrow the policy's behavior even when the reward itself supports a broader range of responses. With stronger optimization, probability mass can concentrate around a small set of reward-favored outputs, producing \emph{preference collapse}~\citep{xiao_algorithmic_2025}, where majority views are further amplified~\citep{shypula2025evaluating,kirk_understanding_2023} and response diversity is reduced~\citep{lakeDistributionalOvertonPluralism2024,khalifa_distributional_2021,cui_entropy_2025}. This is one way in which the downstream policy can sharpen the aggregation choices made by the learned reward: biases in the preference data, once encoded in the reward signal, may be amplified by subsequent optimization~\citep{shapira_how_2026,slocumdiverse}. As a result, the learned policy may fail to represent the full range of normative considerations expressed across populations, domains, and interaction contexts~\citep{santurkar_whose_2023,kirkPersonalisationBoundsRisk2023}.

\runinhead{Multiple Policies.}
A growing number of proposals advocate for \emph{pluralistic alignment}: modeling multiple perspectives in parallel so as to better capture the breadth of human judgments~\citep{sorensen_position_2024, feng2024modular, chakraborty_maxmin-rlhf_2024-1,kim_beyond_2025}. Work in this direction differs in where it relaxes the standard single-reward pipeline. Some approaches retain a single learned reward while limiting how aggressively the policy optimizes against it, reducing the policy-level amplification of the aggregate objective~\citep{slocumdiverse,gx-chen_kl-regularized_2025-2}. Others enrich the reward model so that it can encode preference heterogeneity more directly~\citep{hong2024adaptive,wang2023aligning,xiong_projection_2025}.

A more direct approach represents divergent preferences by training multiple reward models, each matched to a distinct segment of the population~\citep{chakraborty_maxmin-rlhf_2024-1,park_rlhf_2024,chenPALPluralisticAlignment2024,sorensen2025value}. The typical approach clusters annotators, by demographics or inferred preference patterns, and fits a group-specific reward to each cluster. \citet{halpern_pairwise_2025} avoid this intermediate clustering step by fitting a compact reward ensemble whose aggregate pairwise choices match the observed preference $P_x$, so each component induces a coherent policy while the mixture preserves disagreement.

Once such an ensemble is learned, its components can be deployed in several ways to serve pluralistic goals~\citep{sorensen_position_2024}. For example, the system can present multiple outputs as an Overton-style slate~\citep{poole-dayan_benchmarking_2025}, distill them into a consensus statement~\citep{bakkerFinetuningLanguageModels2022,fish_generative_2024,boehmer_generative_2025}, select the policy best matched to a user's stated preference, or sample from the policy mixture. Over repeated use, mixture sampling preserves population-level diversity and counters the tendency of aligned models to recycle a small set of high-reward responses~\citep{kirk_understanding_2023,khalifa_distributional_2021,perez2022red}, which can be particularly valuable in creative workflows.

\runinhead{Personalization.}
At a finer level of granularity, the alignment objective can be personalized to individual annotators, treating each user as carrying a unique reward function~\citep{poddarPersonalizingReinforcementLearning2024}. Since fitting a separate model per user is infeasible under sparse per-user data, these methods place each user's reward in a shared low-dimensional space, recovering a per-user latent variable~\citep{kim_swap-guided_2025} or a weighting over common reward components~\citep{barreto_capturing_2026,bose_lore_2025} from a handful of preferences. Lightweight user embeddings or low-rank adapters then specialize a shared model to the individual. Even so, personalizing to each user risks narrowing their information diet and reinforcing prior views, prompting calls to bound how far personalization should extend~\citep{kirkPersonalisationBoundsRisk2023,kirk_prism-x_2026}.

\runinhead{AI for Social Choice and Democracy.}
The work surveyed in this paper is primarily \emph{social choice for AI}: it uses social choice theory to diagnose how alignment pipelines aggregate heterogeneous feedback, and to design learned rewards or policies that handle disagreement more explicitly. A complementary agenda runs in the reverse direction, asking how AI systems can support collective decision-making, deliberation, and democratic representation. One direction uses AI to help groups deliberate and converge on shared positions, through \textit{generative social choice}~\citep{fish_generative_2024, boehmer_generative_2025}, deliberation mediators such as the \emph{Habermas machine}~\citep{tessler_ai_2024}, formal accounts of common ground~\citep{bakkerFinetuningLanguageModels2022,chooi_finding_2026}, and platforms for large-scale opinion aggregation~\citep{ovadya_generative_2023}. More broadly, progress in social choice for AI and AI for social choice may prove mutually reinforcing, as advances in one direction generate both the concepts and the tools needed to advance the other.

\section*{AI Disclosure}

This work was carried out with the assistance of AI models, which contributed to several parts of the project, including the conceptual development and the writing. These contributions developed through iterative interaction. The authors exercised research judgment at every stage and fully endorse the resulting work.

\bibliographystyle{unsrtnat}
\bibliography{abb,social_choice_rlhf}

\begin{received}
\end{received}

\end{document}